\documentclass[11pt]{article}

\usepackage[utf8]{inputenc}
\usepackage[T1]{fontenc}

\usepackage[margin=1.15in]{geometry}
\usepackage{amsmath,amssymb,amsthm}
\usepackage{booktabs}
\usepackage{tabularx}
\usepackage[expansion=false]{microtype}
\usepackage[hidelinks,breaklinks]{hyperref}
\usepackage{url}

\theoremstyle{plain}
\newtheorem{proposition}{Proposition}
\newtheorem*{propA1}{Proposition A1}
\newtheorem*{propA2}{Proposition A2}
\newtheorem*{corA1}{Corollary A1}

\newenvironment{references}%
  {\clearpage
   \section*{References}%
   \small
   \setlength{\parindent}{-0.4in}%
   \setlength{\leftskip}{0.4in}%
   \setlength{\parskip}{5pt}%
   \par}%
  {\par}

\title{\textbf{Map Users and Mapmakers:}\\[2pt]
\large The Scope of Cognitive Attribution from Acquired Representations}

\author{Yiling Wu\\
\small BridgeM, Inc.}

\date{}

\begin{document}
\maketitle

\begin{abstract}
\noindent An acquired representation can enlarge a system's cognitive repertoire without
transferring the capacities exercised in producing that representation. This paper develops a
framework for specifying that enlargement and its limits. Its central contribution is a five-part
attribution table distinguishing effective tracking, application of acquired structures,
acquisition from explicit specifications, acquisition from identifying observations, and
retention and reuse. Each entry identifies a positive capacity commitment and a further claim
requiring additional support. The argument deliberately grants meaningful content, causal
efficacy, and productive inference, so that its conclusion does not depend on treating
representations as inert encodings. Map and category examples show why even complete application
competence leaves acquisition capacity undetermined, and why acquiring a criterion from its
description differs from finding it in examples. Short formal proofs appear in an appendix. The
framework is applied to Andrew Ng's world-model interpretation of Othello-GPT and to the specific
indicators discussed in contemporary accounts of machine concepts. It preserves demonstrated
recognition, classification, inference, and qualified acquisition while specifying what remains
unestablished about criterion discovery and accumulation. The result concerns the scope of
cognitive attribution rather than the constitutive conditions of concept possession: cognitive
achievements deserve credit for the capacities they establish, without silently importing a
broader repertoire through the labels attached to them.

\medskip
\noindent\textbf{Keywords:} representation; cognitive attribution; representation acquisition;
concept acquisition; large language models; world models
\end{abstract}

\section{Introduction}

A system that can use a map need not be able to make one. Its map can nevertheless support
recognition of locations, comparison of routes, and discovery of connections that the mapmaker
never considered. These are substantial cognitive achievements. The question is how to credit
them accurately while leaving open what the system can establish from observations of an
unfamiliar environment. Acquiring the result of an inquiry and acquiring the capacity to conduct
comparable inquiries are different achievements, even when the result supports extensive
reasoning.

The interpretation of Othello-GPT supplies a concrete reason to ask this question. Li et al.
(2023) provide evidence that a sequence model represents board states and uses them in predicting
legal moves. Andrew Ng (2023), writing in \emph{The Batch}, treats this work as evidence that
language models build world models and understand the world to a qualified extent. His claim
raises a substantive attribution question without asserting that a deployed model can
independently discover new games. Section~6 determines which capacities the underlying evidence
supports. It then applies the same analysis to the specific indicators of conceptual and
principled understanding discussed by Beckmann and Queloz (2026). The framework should earn its
place by clarifying serious, qualified claims, including claims it confirms.

The central thesis has a positive and a restrictive part. An acquired representation, together
with an adequately established operating mechanism, can warrant attribution of applications that
have never been tested, provided they fall within the mechanism's justified domain and
conditions. Complete application competence does not, by itself, entail the capacity to obtain
comparable representations under new acquisition conditions. The common ground is that
representation is not a uniform evidential unit: neither a content label nor a finding of
decodable information carries a fixed package of recognition, inference, learning, and retention
capacities. An attribution must identify which evidential basis licenses its extension. This is a
claim about the scope of warrant, not a claim that systems possessing acquired representations
have an intrinsic limit on what they can do.

The principal contribution is the attribution framework in Table~1. Its five rows distinguish
evidence for tracking, application, acquisition from specifications, acquisition from
observations, and retention and reuse. Each pairs a warranted capacity with a further attribution
that the preceding achievement does not automatically establish. These are dimensions of
evidential support, not grades on a single scale of intelligence. The framework makes a concrete
difference to reporting: a description such as ``the model learned the concept'' must say whether
a criterion was supplied, inferred, or retained, and what its successful use actually
demonstrates.

The argumentative strategy grants the strongest version of the representational achievement
relevant to this question. The examples preserve meaningful content, causal use, generalization
to untested queries, and a genuinely productive interpreter. These grants prevent a familiar but
inadequate reply: that the alleged representation is only information detectable by an external
probe. Even when that concern has been resolved, the acquisition question remains. The resulting
restriction therefore concerns the evidential reach of an effective representation, rather than a
defect built into a deliberately impoverished notion of representation.

There are important predecessors. Chollet (2019, \S II.1.1) distinguishes achieved skill from
skill-acquisition efficiency. Yiu et al. (2024) distinguish transmission from innovation.
Yildirim and Paul (2024) distinguish task-based instrumental knowledge from knowledge supported
by structured world models. Beckmann and Queloz (2026, \S 4.4) distinguish crystallized from
fluid understanding of principles. The present contribution fixes a further question: which
acquisition claims follow once an operative representation and its application competence have
already been granted? It also separates acquisition from a complete description from acquisition
from identifying evidence, a distinction that survives the grant of genuinely new representational
content. Gupta and Pruthi (2025) offer an especially close discussion: their case studies
challenge the sufficiency of state-tracking world models for richer forms of understanding. The
present argument allows rich represented content and productive inference from the outset. Its
distinctive task is to specify warranted application scope and separate it from acquisition under
different inputs, rather than to identify a missing level of explanatory content.

This is an account of the commitments of capacity attribution. It does not make independent
rediscovery, revision, or answerability a universal condition of concept possession. Those
constitutive questions require their own arguments. The present claim can be applied by theorists
who disagree about them, because each must still distinguish the capacities included in a favored
concept attribution from capacities merely associated with it. A theory of concept possession and
an account of evidence for acquisition answer different questions.

Sections~2 and~3 specify the relevant claims and explain why a cognitive product need not
transfer its producer's repertoire. Section~4 supplies the map and category arguments; the
appendix records their formal structure. Section~5 develops Table~1 and the grounds for extending
an attribution. Section~6 applies it to named representational claims, and Section~7 addresses
objections. The conclusion states the resulting scope directly.

\section{What is being attributed}

\subsection{Representation with an established role}

For the attribution analysis, an established representational achievement is a finding about a
system possessing a state or resource with specified content and an established role. This is a
specification of the evidence being assessed, not a definition under which every representation
must already satisfy a rich functional condition. The resource may be a distributed activation
pattern, a learned feature, a structured memory, or another physically realized vehicle. Its role
may include identifying instances, supporting an inference, guiding an action, or preserving a
distinction through a range of transformations. Describing such an achievement already makes some
claims about the system's capacities. The analysis therefore begins by asking which role has
actually been established.

Three descriptions illustrate the difference. A researcher might recover information about
location from a model's activations. A stronger investigation might establish that the system
itself exploits this information to distinguish locations. A further investigation might
establish that the system uses spatial relations to infer routes. These descriptions have
different commitments. In particular, granting the third cannot be followed by treating the
system as a passive store whose only achievement is recoverable information. The argument below
grants the stronger case: representational structure makes a causal difference to successful
inference. This is compatible with the demand that structural representations be causally
relevant and usable, rather than merely resemble their targets (G\l{}adziejewski \& Mi\l{}kowski,
2017).

Content-manipulable representation can be understood here as representation available for
specified operations that respect aspects of its content. The qualification \emph{specified}
matters. Availability for route inference is not identical to availability for reconstructing a
map, diagnosing a mistaken map, or selecting informative observations. Nor does the possibility
of misrepresentation automatically include an ability to identify and correct every
misrepresentation. If a theory makes additional capacities constitutive of its notion of
representation, those capacities belong in its premises. The subsequent evidential question is
whether those richer premises have been established.

The analysis does not settle which theory of representational content is correct. Its primary
audience includes theories on which a map can retain content and support reasoning even when its
possessor lacks the ability to build another map. These encompass familiar possibilities of
content acquisition through teaching, deference, design, and ordinary learning followed by loss
of a learning mechanism. Margolis (1998, pp.~351--354) distinguishes what determines a concept's
content from the mechanisms sustaining its connection to the world. That distinction helps
explain why neither content identity nor a shared referent determines the possessor's entire
cognitive repertoire.

Neutrality has a definite cost. This paper cannot derive one substantive functional lower bound
from content attribution independently of a theory of content. A functional-role theory may make
particular inferential dispositions constitutive; a teleosemantic theory may require a
producer--consumer organization or a history of selected use. Any current capacity inferred from
those commitments requires the theory's premises and, where necessary, a bridge from historical
function to present disposition. Neutrality does not show that no theory can supply a lower
bound. It means that this paper treats the relevant functional premises as premises to be
justified, rather than silently importing them through the word \emph{representation}.

A successful probe is still evidence. It can establish that a variable is recoverable from a
measured state and constrain hypotheses about the information encoded there. By itself, it need
not establish that the system's own processing uses that variable in recognition or inference.
The distinction is between informational evidence and an independently warranted functional
attribution, not between evidence and no evidence. Where a probe result is combined with
background theory or causal evidence, it can contribute to a broader attribution.

\subsection{An occurrence of acquisition and a capacity to acquire}

The statement that a system acquired a representation is initially a claim about an event or
history. A capacity to acquire representations is a disposition across a specified range of
conditions. One successful acquisition may provide evidence of such a disposition, but the two
are not identical. The event might depend on resources no longer available, occur in a narrow
class of circumstances, or involve a larger system than the one subsequently assessed. Evidence
about the event must be related to the current disposition before the latter is attributed.

For this purpose, a capacity claim should specify its bearer, time, task family, admissible
inputs, and resources. A model embedded in a training procedure is not interchangeable with the
same parameterized model after that procedure has been removed. An episode that preserves a
constructed memory is not interchangeable with an episode that resets it. These differences do
not decide which bearer is the philosophically privileged cognitive system. They identify the
bearer of a particular attribution and prevent its conditions from changing unnoticed during an
argument.

The relevant acquisition capacities are also plural. A system may produce new tokens of an
already available representation, recombine familiar elements, construct a new representation
from a description, infer a new criterion from examples, or organize its own inquiry into an
unfamiliar domain. These are distinguishable achievements. In what follows, \emph{acquisition}
concerns obtaining an operative representation appropriate to a new target or target structure
under specified input conditions. This includes modest learning within an existing
representational format. It does not require the first emergence of representation from an
entirely nonrepresentational substrate.

This modest usage prevents an inflated contrast. A route planner that computes a new path is
already generating a new representation in one sense. The thesis is not that it generates
nothing. It is that producing paths from an available map need not establish the capacity to
reconstruct maps from environmental evidence. Similarly, applying an established category to a
novel instance may involve new representational tokens without establishing the capacity to
acquire a different category boundary. The conclusions must track the particular sense of
generation at issue.

\subsection{Inputs are part of the capacity claim}

Supplying a complete rule, supplying examples governed by that rule, and supplying only
opportunities to investigate are different input conditions. A system that succeeds with one
input need not succeed with another. This remains true when an ideal observer could extract the
same information from either. Information may be available in a history without the system
possessing an effective way to extract and use it.

The distinction is not an objection to learning through language. Margolis and Laurence (2011,
pp.~518--521) discuss perceptual, communication-based, and associative routes to concept learning
in rejecting the claim that learning must take the form of explicit hypothesis testing.
Acquisition from testimony or a supplied definition may confer real understanding and new
abilities. It simply does not confer, by that fact alone, every ability required to obtain the
same result through independent investigation. Input-sensitive attribution recognizes more than
one kind of acquisition rather than reserving the term for the most demanding kind.

\section{What a representational product can transmit}

The motivating asymmetry is between benefiting from a completed cognitive achievement and
possessing the capacities exercised in completing it. A map can make the results of surveying
available without making its user a surveyor. A taxonomic description can transmit a useful
distinction without reproducing the original investigation. The recipient can nevertheless do
more than repeat the producer's outputs. Once incorporated into an effective inferential system,
the resource can support consequences that neither party previously considered. The resulting
capacity is an achievement of the recipient as equipped with that resource.

The asymmetry therefore has two directions. The acquisition of a product does not guarantee
inheritance of its producer's capacities. The product's external origin does not guarantee the
absence of capacities in its recipient. A compiler, a training procedure, or a learner might
transform received material into a new generative procedure. Whether this occurs is a further
matter concerning the resulting system. It is not settled by classifying the input as a product
of someone else's cognition.

Chollet's distinction between skill and skill acquisition is directly relevant to the first
direction (2019). The present argument also emphasizes the second: withholding a general
acquisition capacity leaves room for substantial recognition and reasoning capacities. Likewise,
the distinction between transmission and innovation in Yiu et al. (2024), and the
cultural-technology account developed by Farrell et al. (2025), draws attention to how systems
can benefit from accumulated human achievements. The question pursued here is what the resulting
resource enables its possessor to do. Its cultural provenance does not settle that question in
either direction.

There is a further constraint on arguments from the prerequisites of generation. Suppose
acquisition capacity $G$ requires a capacity $P$, and possession of representation $R$ does not
entail $G$. It does not follow that possession of $R$ fails to entail $P$. A simple logical
counterexample makes this clear: let $G$ require both $P$ and an independent capacity $Q$, while
the established use of $R$ already requires $P$. Then $R$ can guarantee $P$ without guaranteeing
$G$.

Spatial discrimination might, for example, be needed both for making a map and for using one. The
failure of the inference from map possession to map acquisition would then leave the attribution
of spatial discrimination untouched. Each proposed prerequisite must be considered separately; no
general list of absent capacities follows from a single failed entailment.

A valid argument in the other direction is available. If a further capacity $C$ requires the
specified acquisition capacity $G$, a system possessing $R$ but lacking $G$ also lacks $C$. This
can establish that $R$ does not entail $C$. The necessary premise $C \Rightarrow G$ must,
however, be defended for the task in question. Independently investigating an unknown category
plausibly includes some capacity to establish its criterion from evidence. Classifying instances
under a criterion already acquired need not. The examples in Section~4 make this distinction
explicit.

\section{Application can be complete while acquisition remains open}

\subsection{The map user and the mapmaker}

\begin{proposition}[Application and acquisition]
Within a finite family containing at least two targets distinguishable by application queries,
suppose that each target has a finite, nonempty set of admissible observation histories, each
identifying its target. Two systems can possess the same acquired representation and the same
complete application interpreter, agree on every application query about the acquired target, yet
differ in whether they can acquire the family members from those histories. Exhausting the
application queries therefore need not distinguish their acquisition capacities. The formal
assumptions and proof are in Appendix~A2.
\end{proposition}

Consider a network of six named locations connected by passages. Two systems possess the same
accurate adjacency matrix and run the same graph-search procedure. Each can answer whether two
locations connect, whether one is reachable from another, and which route is shortest, with a
fixed convention for equally short routes. The answers are computed from the represented
connections. They are not retrieved from a list of previously answered questions. Both systems
can discover an unasked-for route through the network, including one never considered by whoever
produced the map.

Now give each system reports about a different environment on the same six locations. The reports
identify the presence or absence of all fifteen possible undirected connections, in any order.
One system records them in a new matrix. The other retains its old matrix, although it remains
perfectly able to reason with that matrix. The first can acquire an operative map from the
reports. The second cannot acquire the new map by this route. The information is available to
both, and the required learning procedure is straightforward. Their difference is not produced by
asking for an inference that the observations cannot support.

Before those reports are processed, the two systems agree on every application question about
their acquired map. This agreement is exhaustive: testing every connection, every reachability
claim, and every shortest route would not distinguish their acquisition capacities. More
difficult or more numerous questions about the old network cannot reveal which procedure handles
evidence about another network. An application evaluation does not become an acquisition
evaluation merely by becoming comprehensive. Once the evaluation supplies a new environment and
requires an operative new map, it has changed the achievement being examined.

The example allows the initial map to have a normal representational history. Both systems may
first have acquired it by learning, after which the acquisition mechanism of one is disabled
while its map and search procedure remain intact. Alternatively, both may have received an
accurate map from another source. The argument needs the retained content and use to be possible,
not a theory on which any arbitrary physical state counts as a representation. It also leaves the
implementation of learning open: a single integrated device could realize the relevant
dispositions. Distinct capacities need not occupy distinct modules.

The conclusion is a failure of entailment, not a finding that a particular neural model resembles
the retaining system. An acquired structural representation, even together with complete
application competence, leaves at least two possible acquisition profiles open. Background
knowledge about a particular system may favor one. The point is that the representation and
application profile alone do not settle the matter. Proposition~A1 and its corollary formalize
this construction.

\subsection{A family of category criteria}

The same argument applies to acquiring a criterion rather than a spatial map. Consider devices
with four binary features: a trigger, a gate, a shield, and a charge. A target category is fixed
by one of six possible criteria: trigger and charge; (trigger or gate) and charge; trigger and no
shield; (trigger or gate) and no shield; charge and no shield; or at least two of trigger, gate,
and charge. The family thus varies which features matter, whether an apparent prerequisite is
dispensable, and whether conjunction, disjunction, or a threshold determines membership.

Both systems begin with the criterion ``trigger and charge.'' They can classify all sixteen
feature combinations, recognize shared membership, and apply the criterion to devices they have
never encountered. Give both systems labelled observations from a new target category. One
maintains the set of candidate criteria consistent with the evidence and adopts the remaining
criterion when the evidence identifies it. The other continues to use its acquired criterion.
Sufficient observations always exist because the six criteria disagree somewhere in the finite
feature space; a complete labelled set is sufficient, although identifying subsets can be
smaller.

Suppose the new criterion is ``(trigger or gate) and charge.'' The evidence can show that a
charged device without a trigger still qualifies if it has a gate. Learning this criterion
removes the trigger as a necessary condition while retaining charge as necessary. It changes the
classification of both shielded and unshielded devices with a gate and charge but no trigger. The
change thus has consequences for multiple objects, including objects not used as identifying
examples. Appendix~A4 specifies a six-observation set that identifies each criterion in the
family and leaves ten feature combinations unobserved.

This is more than changing one object's recorded membership. The acquiring system adopts a rule
whose consequences extend beyond the observations used to select it. Nevertheless, both systems
were completely competent with the initial rule. Neither the breadth of that classification
competence nor its extension to previously unseen instances guarantees the capacity to establish
a new criterion. The family is deliberately modest: it shows that the separation arises even when
all relevant features and the space of possible criteria are available. No demand for the
creation of entirely new conceptual primitives is needed.

\subsection{Supplying a criterion and finding one}

\begin{proposition}[Acquisition under different inputs]
In the same finite setting, reliable acquisition and use of every target from its complete
specification does not entail reliable acquisition and use of those targets from observations,
even when the admissible observations identify the target. The two capacities can differ while
the representation format and application interpreter remain fixed. Appendix~A3 supplies the
proof.
\end{proposition}

Now equip the second system with a parser that installs any of the six criteria from a complete
verbal or symbolic specification. It can receive a novel criterion, construct the relevant
operative representation, and apply it correctly to every device. It genuinely acquires something
new. Its procedure can nevertheless fail to select a criterion from labelled observations. A
parser for a description and a procedure for identifying which description fits the evidence
accomplish different tasks.

The distinction does not depend on missing information. The same observation sets that identify
the criterion for the first system are available to the second. What the latter lacks is a
procedure for extracting their significance. Nor does successful instruction-following deserve to
be dismissed: interpreting the instruction and applying its consequences may itself require
substantial cognition. The justified attribution is acquisition under specification-providing
conditions. Acquisition from observations is an additional claim. Proposition~A2 gives the
corresponding formal counterexample.

The map and category cases establish two of the boundaries in Table~1. Complete application
competence does not establish acquisition from fresh evidence, and genuine acquisition under one
form of input does not automatically establish acquisition under another. They leave positive
attributions intact while identifying exactly where an argument needs something further.

\section{The attribution framework}

The answer to the scope question has three parts. \textbf{Within application}, a representation
plus an adequately established mechanism can support untested cases throughout the mechanism's
justified domain. \textbf{Across acquisition conditions}, even exhaustive application competence
supplies no deductive warrant by itself; an extension requires grounds that connect the system to
the new acquisition demand. \textbf{Without an established role}, a content label or decodability
result does not independently warrant the richer functional attribution, although it can remain
informative evidence. These are limits on particular inferences, not an assertion that
application evidence has zero probabilistic relevance to learning.

Table~1 makes these claims usable. Its first column records the established basis; the second
states the positive attribution; the third identifies an additional commitment that is not
automatic. An adequate basis may be mechanistic, behavioral, or explanatory. Thus the table
neither treats a few correct answers as a disposition nor requires every future application to be
tested separately. The five rows are useful distinctions among evidence conditions, not an
exhaustive taxonomy of representations. Propositions~1 and~2 establish the two acquisition
boundaries. Section~5.2 gives separate constructions for inquiry and retention.

\bigskip
\noindent
\begingroup
\small
\begin{tabularx}{\textwidth}{@{}XXX@{}}
\toprule
\textbf{Established achievement} & \textbf{Warranted capacity attribution} &
\textbf{Further attribution requiring support} \\
\midrule
Causally effective tracking across presentations &
Recognition and same-tracking over the supported variations &
Acquisition of new target distinctions from fresh evidence \\
An operative criterion evaluator or structure-based procedure with established correctness
conditions &
Classification and inference throughout the supported domain, including untested instances and
queries &
Discovery or reconstruction of that criterion or structure \\
Repeated acquisition from complete specifications &
Acquisition and application under specification-providing conditions &
Acquisition from examples or observations without the specification \\
Repeated acquisition from identifying observations &
Observation-based acquisition over the supported target family &
Selection of informative observations and autonomous inquiry \\
Demonstrated retention and reuse across stated interruptions or contexts &
Retention and reuse under those temporal and resource conditions &
More enduring or broader accumulation without those conditions \\
\bottomrule
\end{tabularx}

\medskip
\noindent\textbf{Table 1.} Scope of attribution from representational achievements. The
rows distinguish evidence conditions; they are not successive grades of intelligence. A further
attribution may be supported by an established mechanism, performance under the relevant demands,
or a justified explanatory inference. Its placement in the third column does not imply its
absence.
\endgroup
\bigskip

\subsection{Why the positive scope exceeds observed performance}

The positive proposal is not merely that an established capacity may be renamed. Mechanistic
evidence can describe an encoding and an algorithm without listing the answers to every possible
application; an argument about how that algorithm operates can then warrant the unlisted answers.
The inference uses the mechanism's organization and correctness conditions, rather than assuming
the full behavioral profile that it concludes.

Consider a concrete realization of the map user. Its accurate adjacency matrix is read by
breadth-first search: the procedure places a starting location in a queue, examines the neighbors
of each queued location, and records a location only on its first discovery. Every recorded
predecessor link is a genuine edge, and every location reachable in a given number of steps is
discovered by the time that search depth has been completed. Induction on path length therefore
establishes reachability coverage. Because discovery proceeds by increasing distance, the first
predecessor chain to a location is a shortest route. On a finite graph, the procedure terminates.
These facts warrant answers to source--destination queries that have never been run, provided the
stored map is accurate and the implementation and resource conditions satisfy the argument. The
justification explains the extension instead of merely describing previous success.

The resulting scope is broad but determinate: reachability and shortest-route questions about the
represented unweighted network. It does not automatically cover time-dependent travel costs, an
inaccurate map, or learning a new network from observations. Even a procedure correct on every
supplied map need not construct a map. This is how a positive projection within application can
coexist with Proposition~1's non-entailment across acquisition. For approximate neural
mechanisms, an explanatory account may justify a fallible projection over a specified
distribution rather than an exact guarantee over every input; the distinction between these
strengths of warrant should be reported.

The same pattern applies to Table~1's other rows. Evidence of effective tracking warrants
recognition over supported presentation changes. An established criterion evaluator warrants
classifications of previously unencountered feature combinations. Neither permission depends on
those very instances having appeared in an evaluation, although it does depend on grounds for
extending the mechanism's operation to them. Information recoverable by a probe without such
grounds supports a narrower claim about encoding.

The acquisition rows credit mechanisms and dispositions with their input conditions intact. A
procedure that interprets new specifications and puts them to use warrants
specification-conditioned acquisition. A learner that establishes criteria from identifying
observations warrants observation-conditioned acquisition within its supported family. A
mechanism may justify either attribution beyond the cases already run. Retention and reuse
likewise extend over the interruptions, durations, and resources for which there is an adequate
basis. These grants add substantive commitments; the table's third column prevents those
commitments from expanding unnoticed.

\subsection{Why extending the attribution requires more}

The third column identifies demands that differ from those already established. In the first two
rows, the further task requires obtaining a distinction or structure that the original
application task presupposes. Section~4 shows that keeping all original application answers fixed
can leave success on that further task variable. This is why ``test more applications'' is not,
by itself, an answer to the acquisition question.

In the third row, removing a complete specification requires the system to establish a criterion
that it previously received. Even if the observation set uniquely determines the answer, a
procedure for extracting that answer remains necessary. Success after receiving the answer in a
usable description cannot demonstrate that procedure. This is the substantive reason for the
specification--observation distinction, rather than an arbitrary division of input formats.

The fourth row adds evidence selection. Equip two copies of the six-criterion learner with the
same ability to request the label of any of the sixteen feature combinations. Both infer
correctly when an identifying dossier is supplied. Under a common budget of sixteen requests, one
selects every combination once; the other repeatedly requests the all-zero combination. Every
criterion assigns that combination the same negative label, so the second policy never
distinguishes the targets, whereas the first obtains an identifying dataset. The systems share
the passive learning procedure, measurement channel, and budget but differ in successful inquiry.
Observation-based learning therefore does not entail selecting informative observations; the
separation is established by a construction, not by an omission from Proposition~1.

The fifth row concerns a different transition. Take two systems with the same acquisition
procedure, acquired criterion, and within-episode answers. At an episode boundary, one preserves
an accessible copy of the criterion; the other clears every acquired copy and returns to its
initial state. Supply no new specification or observations after the boundary. A later query that
distinguishes the acquired criterion from the initial one can then receive different answers.
Agreement throughout acquisition and application before the interruption does not determine
retention afterward. This construction changes the persistence condition explicitly; it does not
infer that externally stored information is less genuinely retained or that temporary acquisition
never occurred.

These differences explain why the table is not a single ascending hierarchy. A system may
preserve a received map for years without being able to make one; another may build an accurate
map and then lose it when its working memory is cleared. Likewise, an expert user may reason
extensively from an acquired theory while having narrower acquisition capacities than a simple
learner in its small domain. The strength of an attribution must be specified by what varies, not
summarized by one undifferentiated capacity label.

\subsection{How a stronger attribution can be warranted}

Additional support can take several forms. An operative mechanism may be independently known to
acquire representations across the relevant target family. Repeated success under the relevant
acquisition conditions may support the corresponding disposition. A broader explanatory argument
may connect the established performance to a reusable learner. Direct observation of every
internal step is not required. What matters is that the grounds support the additional task
rather than restating success on the original one.

The argument thus limits entitlement from specified evidence, not the capacities a representer
may possess. A learned representation can encode a procedure, including a learning procedure, and
a system can execute it. Verified execution would support that acquisition attribution.
Similarly, compression or instruction can give rise to a new generative competence rather than a
store of answers. No fixed ceiling follows from describing a resource as acquired. What fails is
an automatic transfer from the product's cognitive status to every capacity associated with its
production.

\section{Applying the framework to published attributions}

\subsection{Andrew Ng on Othello-GPT and world understanding}

Li et al. (2023) trained a sequence model on Othello transcripts and investigated board-state
information in its activations. Their interventions provide evidence that this information
matters to move prediction, including interventions involving board configurations unreachable by
ordinary play. This supports more than a correlation accessible only to a researcher. Li's
accompanying explanation nevertheless explicitly declines to infer understanding or intelligence
from the crow analogy used to present the result (Li, 2023, ``A thought experiment'').

Ng's original discussion makes an explicit interpretive claim. He writes that the Othello work
supports the conclusion that ``LLMs build world models,'' and connects world-model complexity to
a qualified attribution of understanding (Ng, 2023). He explicitly describes a process of
training on move sequences and acknowledges the lack of an agreed scientific test of
understanding. It would therefore misrepresent his position to read it as a demonstrated ability
of the deployed checkpoint to discover arbitrary new rules. The relevant question is what scope
the positive interpretation warrants.

Under Table~1, the intervention evidence supports an acquired board-state representation with a
demonstrated role in prediction: an application achievement in the second row. The observation
that training produced this result is also evidence about that particular training process. These
grants need not be withdrawn. They do not determine whether the deployed checkpoint can infer
another game's rules from examples, select informative interactions, or preserve a newly acquired
representation across resets. Such extensions change the target family, bearer, input conditions,
or temporal conditions. Calling the resource a world model does not supply the missing bridge.

This analysis gives Ng's interpretation determinate content without treating it as an explicit
universal implication. If world understanding denotes the established representational use, the
framework confirms that attribution in its specified domain. If the description is taken to
license a further acquisition capacity, that commitment requires additional grounds. The
distinction is between the justified scope of the evidence and a possible extension of its
description, not between the competence of the original investigators and the looseness of public
commentary.

Roth (2023) offers a revealing secondary illustration when he describes the result as ``some
understanding of the world'' and adds ``kinda like smart kids learning.'' A possible effect of
this analogy is to invite a reader to carry capacities for learning new distinctions into the
attribution. That is an interpretive risk of the wording, not an allegation about Roth's intended
thesis or a measured claim about readers' responses. The scope analysis does not need the analogy
to bear the argumentative weight: the controlled evidence and Ng's explicit world-model
attribution already provide a substantive case.

Gupta and Pruthi (2025, \S\S3--4) also scrutinize the passage from world models to understanding,
asking whether representations of states and transitions capture explanatory insights. The
present question remains even if those insights are represented and used. A system can reason
productively with a richer acquired model while lacking a particular way to obtain another one.
Enriching what is represented and establishing how new representations can be acquired are
consequently distinct responses to distinct attribution questions.

\subsection{A worked assessment of Beckmann and Queloz's indicators}

Beckmann and Queloz distinguish conceptual, state-of-the-world, and principled understanding, and
separate crystallized from fluid understanding of principles (2026, \S\S2--4). Table~1 adds a
different classification: what particular evidence warrants attributing under particular inputs.
Three indicators make the comparison concrete.

\paragraph{Entity features: support for the first row.}
Their discussion of conceptual understanding includes a feature associated with the Golden Gate
Bridge. The underlying study identifies responses across languages and images and reports changes
in output when the feature is amplified (Templeton et al., 2024). The relevant evidence thus
combines cross-presentation sensitivity with a causal intervention, rather than relying only on
probe decodability. Table~1 locates the positive claim in recognition and target-related use over
the supported variations. Amplification, especially outside the normal activation range, does not
by itself establish an unrestricted recognition mechanism or every inference about the bridge;
causal interpretation requires the other evidence to support that extension. Nothing in this
attribution requires denying the represented target. The further capacity to establish an
unfamiliar entity distinction from new observations would need evidence of the fourth-row kind.

\paragraph{Modular addition: support for the second row.}
Nanda et al. (2023) study a small transformer trained on addition modulo 113, using 30\% of the
input pairs for training and evaluating the remaining pairs. They identify a Fourier-based
procedure and test its causal relevance through targeted ablations. The evidence supports an
operative computation, including application to inputs absent from training. Its place in
Table~1 is therefore the second row, with a specified arithmetic domain. This is a positive
inferential attribution, not a retreat to mere representation. It also illustrates why mechanism
matters: the analysis explains how generalization is produced. However, explaining this
computation does not establish a procedure for identifying a different operation from a few
examples at inference time. The study of circuit formation supplies additional evidence about
acquisition during the reported training process, whose bearer and conditions must be kept
distinct from the trained checkpoint's later capacities. This matches Beckmann and Queloz's
treatment of the example as crystallized understanding (\S 4.4).

\paragraph{Fluid understanding: evidence relevant to the fourth row.}
In \S 4.4, Beckmann and Queloz ask whether models can identify a new principle at inference time.
They turn to ARC, where example input--output pairs are supplied and the test taker must produce
an output for a further input; they do not infer a positive answer from the modular-addition
circuit. Their cited technical report describes both model-guided program synthesis and test-time
training (Chollet et al., 2025, \S\S2--3). In the synthesis approach, candidate programs are
generated, checked against the demonstrations, and used on a test input. The target program is
not supplied as a complete rule description. This is evidence relevant to observation-conditioned
acquisition and application, rather than merely to applying a given rule.

The precise credit belongs to the evaluated procedure. Successful synthesis provides positive
evidence that the system comprising the model, search, selection, and interpreter can obtain a
task-solving rule from examples in the tested range. It does not isolate all those operations in
the model component. Nor do a few ARC examples uniquely identify a rule over an unrestricted
hypothesis space: the assessment depends on the task family, prior repertoire, novelty controls,
and permitted computation. Appendix~A's finite identifying histories establish a logical
separation; they are not an assumption that every benchmark item provides logical identification.
The empirical attribution is correspondingly defeasible and bounded. It does not establish
autonomous selection of environmental observations or retention across episodes, neither of which
is demonstrated simply by solving from the supplied pairs.

The worked assessment consequently confirms the crystallized--fluid distinction and makes its
application more precise. The second-row evidence warrants arithmetic application; the fourth-row
task probes acquisition from examples, with positive credit where the complete evaluated system
succeeds. Acquisition from a supplied novel rule would instead fall in the third row, even if it
occurs at inference time. Thus neither novelty of content nor inference-time construction alone
settles the acquisition condition. The framework does real work in a serious source without
requiring its authors to have made an invalid inference.

\subsection{Vector formats and operative procedures}

Piantadosi et al. (2024) argue that vectors, together with appropriate dynamics, can meet
familiar demands on conceptual representation, including procedural knowledge. That format-level
proposal is compatible with every row of Table~1. Whether a particular system realizes a row
depends on its operative capacities, not the expressive resources of the format alone.

A vector representation of a learning procedure can indeed contribute to an acquisition capacity
if the system executes the procedure. Such a case belongs in the table's positive columns.
Conversely, evidence that a format can encode procedures does not show which procedures a trained
system implements. Section~4 grants a shared effective interpreter and still separates
acquisition profiles. Format adequacy and realization in a particular system therefore answer
different questions.

The consequence is symmetric. Failure to discover a new criterion does not refute vector
representation as a medium for concepts. The adequacy of vectors as a medium does not establish
that a particular model can discover it. A substantive attribution identifies what the model does
with the represented content and under which input conditions it can obtain further content.

\subsection{New representations through linguistic input}

Coelho Mollo and Milli\`ere (2026) discuss transient grounded representations through in-context
learning. Their new-game example supplies a complete statement of the rules; the proposed
grounding route is presented as a possibility requiring further mechanistic
investigation.\footnote{The new-game example and its qualification are cited from the authors'
publicly available manuscript, arXiv:2304.01481v3, \S 6.2.3. The reference list records the
subsequent journal publication.} This provides a useful application of the third row of Table~1.
If the model constructs and uses a new representation from those rules, it genuinely acquires
operative content under specification-providing conditions. The additional question is whether it
can obtain comparable rules from observations of play.

Calling the first achievement instruction-following does not diminish it. A sufficiently rich
instruction may require conceptual interpretation and substantial inference. It still supplies a
result that a discovery task requires the system to obtain. The distinction remains relevant
after the strongest reasonable interpretation of the new representation has been granted.

Lederman and Mahowald (2024) emphasize the challenge of novel reference for accounts that explain
model meaning entirely through inherited textual meaning. Their interpretationist argument
provides a reason to recognize further referential achievements when the relevant behavior
supports them. Table~1 allows that extension without turning new reference into evidence of every
way of discovering, investigating, or retaining a referent's structure.

Yildirim and Paul's (2024) distinction between instrumental and worldly knowledge addresses
whether an operative representation captures structure in the world. The map argument begins
after such structural and functional resources have been granted. Its question is what further
acquisition capacity follows. A richer account of what is represented can strengthen the positive
application attribution while leaving that question open. The two distinctions concern different
claims and should be used together.

\section{Objections and empirical relevance}

\subsection{Is there an error to correct}

A critic may accept the examples but ask who asserts the implication they reject. Ng's world-model
interpretation is an explicit, qualified cognitive attribution, not a universal claim about every
form of learning. Roth's analogy illustrates a possible extension in how an attribution is
received, not evidence of a developed philosophical mistake. The paper should accordingly not
claim to have exposed a widespread deductive fallacy on the strength of either text.

The stronger answer is the worked assessment in Section~6.2. It identifies the inputs,
mechanisms, and outcomes of actual indicators, assigns them to the table, and states both
positive and unestablished capacities. The result confirms part of Beckmann and Queloz's account
and distinguishes acquisition conditions and bearers within it. This is useful even when the
original attribution is correct: a framework for scope should specify what a qualified claim
commits its user to, rather than depend for its value on finding an incautious opponent. The
separate logical constructions explain why the additional questions are substantive.

Nor is the positive proposal simply ``attribute whatever has been established.'' Section~5.1
supplies the missing inferential step: evidence identifying a mechanism, together with a
correctness argument under specified conditions, warrants application claims not already
contained in a list of observed responses. The negative constructions then show why the same
reasoning does not automatically extend to acquisition. The contribution combines justified
projection with specified points at which a further bridge is needed.

The examples might also be charged with defining one system to fail. For an empirical claim that
a real model lacks acquisition, that objection would be decisive. For a proposed entailment, the
appropriate test is whether the stipulated representation and effective use can coexist with
different acquisition profiles. Ordinary maps and graph-search procedures demonstrate that they
can. Preserving complete application competence is what gives the example its force: its
conclusion survives the strongest application evidence, not just an impoverished test. The formal
appendix records this limited logical result without making its elementary proof the paper's
principal contribution.

\subsection{Representational content already requires broader abilities}

A functional-role theory may hold that representing a category requires more than classification,
perhaps including systematic inference or responsiveness to reasons for revision. The present
account can grant these requirements. Their consequences belong in the initial achievement
description. The question then becomes whether the particular further acquisition capacity also
belongs there. Adding some inferential and corrective dispositions does not establish every other
disposition without argument.

Camp (2009) distinguishes commonality of representational content from the systematic abilities
characteristic of conceptual thought. Butlin (2023, \S 3) develops related distinctions in
discussing machine concepts, including the difference between possible recombinations elicited by
stimuli and more substantive representational abilities. These arguments show why representational
claims can carry genuine functional commitments. They do not make the acquisition of every new
category a consequence of every such commitment. The current analysis asks what remains open
after the relevant systematic abilities have been granted.

A constitutive thesis that concept possession requires answerability for misuse would add a
substantive condition of possession. Whether that thesis is correct is a separate issue from
whether a possessor can acquire new criteria. A system might preserve and appropriately revise
the use of an acquired criterion while having a restricted repertoire of ways to obtain others.
Conversely, selecting a criterion from labelled examples need not establish every form of
normative answerability. This paper neither derives nor rejects the constitutive thesis. It
requires an attribution invoking it to establish its additional condition, while crediting the
narrower capacities already supported.

\subsection{Different histories and external support do not settle capacity}

A further objection concerns human comparison. Most people acquire concepts through language and
instruction without repeating the discoveries that made those concepts available. They can still
understand and reason with them. This is a reason to reject any requirement that a concept
possessor reproduce its original acquisition history. It is also a reason to retain the present
distinction: understanding a received result does not establish competence in every investigation
that contributed to it.

Similar caution applies to external support. Milli\`ere and Rathkopf (2026, \S\S3.1--3.3)
distinguish performance limitations imposed by auxiliary demands from incapacity, and warn
against rejecting unfamiliar implementations because they differ from human mechanisms. A model's
failure after its memory is removed may reflect the removal of a condition for exercising a
capacity. It does not by itself show that the capacity was performed by the memory. A component
that preserves an acquired representation need not be the component that generated it.

Three cases should therefore be distinguished when interpreting scaffolding. The external
resource may provide a completed target representation, preserve a representation produced by the
model, or carry out part of the inference needed to obtain it. These differences bear on
attribution, but a performance gap alone does not discriminate among them. The present argument
uses such cases to specify input and resource conditions; it does not infer a location of
cognition from a score difference. Nor does it require that all legitimate cognitive capacities
be exercised without external resources.

\subsection{Learning can survive compression or fixed parameters}

The claim that a representational product can compress earlier cognitive work does not entail
that a compressed system lacks the corresponding capacity. A learned procedure can implement the
same relevant transformations through different internal steps. If it succeeds across the
required changes, its competence should be recognized. Hsieh et al. (2023) use explanations
produced by a larger model as training signals for smaller models and report improved task
performance. Such results illustrate why receiving a product of another system can help produce a
new competence, rather than merely adding stored outputs.

Fixed parameters also do not imply the absence of acquisition. In a controlled regression
setting, von Oswald et al. (2023) construct and study attention mechanisms that perform
transformations corresponding to gradient-based learning within the forward pass. This is
evidence about specified models and tasks, not a universal explanation of in-context learning. It
nevertheless shows why parameter updating cannot serve as the definition of every kind of
acquisition. Relevant states may change within an episode while the longer-term parameters remain
fixed.

The proper response to these possibilities is to identify the operative learning procedure and
its domain. If the procedure acquires task-specific representations from examples, the
corresponding capacity should be attributed. Whether the result persists, transfers to another
task family, or supports independent inquiry is a further question. This is a substantive
expansion of attribution, not an exception to the account.

\subsection{What existing experiments contribute}

The empirical literature already distinguishes several components that broad capacity labels can
obscure. Qiu et al. (2024) study iterative hypothesis refinement, including hypothesis proposal,
selection, and application to instances. Their use of a symbolic interpreter helps distinguish
the contribution of proposed rules from the ability to apply them. These findings supply evidence
about particular systems and procedures; they do not justify treating all models as possessors of
content without any generative process. Indeed, hypothesis proposal is itself a positive
contribution that the attribution should preserve.

Tyen et al. (2024) distinguish finding an error from correcting a solution after its location is
supplied. The analogous inference at issue is from successful correction with diagnostic
information to successful production of that information. Their results illustrate a difference
in input conditions resembling the one formalized in Proposition~A2. They do not establish that
the same difference holds across every kind of conceptual revision or every model.

ConceptEdit examines conceptual definition editing and its effects on instance-level knowledge;
the reported difficulties concern the methods and models studied (Wang et al., 2024). RelEdit
extends evaluation to related conceptual and instance-level reasoning and reports advantages for
a memory-based in-context editing approach (Niu et al., 2025). These studies provide reasons to
distinguish the successful installation of a description, its use in judgments, and the
procedures needed to revise related knowledge. Neither successful editing nor unsuccessful
propagation alone determines who can discover the revised criterion from evidence.

These studies are relevant to the theoretical argument because they examine changes in what is
supplied and what the system must accomplish. They need not be treated as a single empirical
proof of a general representational deficit. The constructions establish a logical possibility.
Empirical comparisons investigate which distinctions are realized, to what degree, and under
which conditions. Keeping these roles separate allows positive findings to strengthen the
resulting attributions rather than threatening the framework.

\subsection{A possible alternative does not defeat the best explanation}

Finally, the existence of a retaining system does not refute an inference to the best explanation
that a real model has a reusable acquisition capacity. If a mechanism or disposition explains a
broad range of successful acquisitions, then the hypothesis may be well supported even though a
logically possible alternative exists. Ordinary scientific attribution does not require
eliminating every imaginable alternative.

The present argument instead identifies a reason that certain alternatives are relevant to a
specific inference. The systems in Proposition~A1 agree on the entire type of achievement used as
the premise. Evidence restricted to that type does not discriminate their acquisition procedures
by direct testing. A defensible explanatory inference must accordingly draw on something more: a
known relation between application and learning mechanisms, a developmental history, successes
under acquisition demands, or other background constraints on the system class. Such premises may
be available. Their availability, rather than representation possession by itself, determines the
strength of the further attribution.

This gives the analysis a constructive stopping point. It is unnecessary to demand an experiment
for every philosophical claim, but it is equally unnecessary to treat the results of a thought
experiment as a diagnosis of actual machines. The theory specifies what the evidence needs to
bear on. It then recognizes the stronger capacity whenever an adequate explanatory basis is
present.

\section{Conclusion}

The scope of attribution from acquired representations is conditional and multidimensional. An
operative representation can support recognition, classification, and productive inference;
understanding its mechanism can warrant applications that have never been tested. In the map
case, an accurate matrix and an established breadth-first search procedure warrant reachability
and shortest-route answers throughout the represented unweighted network, under the stated
implementation and resource conditions. This is a substantive positive range, not merely a
redescription of recorded successes.

That application range does not entail an acquisition range. Even agreement on every application
query can coexist with different capacities to construct a map or identify a criterion from fresh
observations. Genuine acquisition from a complete specification also leaves observation-based
acquisition open. Obtaining a criterion from a supplied dossier does not entail selecting an
informative dossier, and within-episode acquisition does not entail later retention. Each
extension needs support for the changed demand. These are boundaries of deductive warrant from
specified premises; they neither deny the stronger capacity nor prohibit its attribution on
independent or explanatory grounds.

Conversely, a content label or a probe finding without an established operating role does not
independently warrant those application capacities. It can still support an informational claim
and contribute to a broader explanation. Because the paper remains neutral among content
theories, it supplies no universal functional lower bound from content alone. A theorist who
derives such a bound must make the relevant content-theoretic and functional premises explicit.
There is therefore no single capacity package, or fixed upper limit on possible abilities,
attached to the word \emph{representation}.

Table~1 gives the practical verdict. Credit recognition within supported presentation changes;
credit application within justified mechanistic or dispositional scope, including untested cases;
credit acquisition under the specification-providing or observation-providing conditions actually
supported; and credit retention and reuse under established temporal and resource conditions.
Credit autonomous inquiry when there are grounds for the additional evidence-selection capacity.
The framework confirms the trained Othello representation's demonstrated predictive role, the
modular-addition circuit's application competence, and observation-conditioned acquisition where
the evaluated ARC procedure supports it. It does not transfer these achievements to another
bearer, input condition, or temporal setting without an argument.

These conclusions concern what particular representational evidence warrants saying about a
system. They leave constitutive theories of concept possession open while replacing an
undifferentiated attribution with explicit commitments about what the system can do, over which
variations, and with which support.

\clearpage
\appendix
\setcounter{equation}{0}
\renewcommand{\theequation}{A.\arabic{equation}}
\section{Formal separation results}

The appendix records the logical structure of Section~4. Exact success is used to establish
counterexamples to entailment, not as a threshold for attributing fallible human or machine
abilities.

\subsection*{A1\quad Setting}

Let $H$, $X$, and $Y$ be finite sets of targets, application queries, and answers. Write
$a_{h,x}$ for the correct answer, and assume
\begin{equation}
h \neq h' \ \Longrightarrow\ \exists\, x \in X : a_{h,x} \neq a_{h',x}.
\end{equation}
Let $R$ be a finite state set, $e : H \to R$ an injective encoding, and $v : R \times X \to Y$ an
interpreter satisfying
\begin{equation}
v\bigl(e(h), x\bigr) = a_{h,x} \qquad (h \in H,\ x \in X).
\end{equation}
This grants the interpreter's competence; it does not infer competence or semantic content from
an arbitrary encoding. Fix $r_0 = e(h_0)$. Let $D$ be the set of finite observation strings, and
$D_h$ a nonempty finite set of admissible histories for each target. Require
\begin{equation}
h \neq h' \ \Longrightarrow\ D_h \cap D_{h'} = \varnothing .
\end{equation}
Thus admissible histories identify their target. For a total computable procedure
$U : R \times D \to R$, let $S_U = (r_0, v, U)$. Define acquisition success by
\begin{equation}
G_D(S_U) \iff \forall h \in H\ \forall d \in D_h\ \forall x \in X :
v\bigl(U(r_0, d), x\bigr) = a_{h,x}.
\end{equation}
At the application stage the interpreter uses the current representation without updating it in
response to a query. This is a possible system class, not a proposed architecture of all
cognition. A common finite resource bound suffices for all admissible cases in the constructions.

\subsection*{A2\quad Application does not entail acquisition}

\begin{propA1}
If $|H| \geq 2$, there are systems sharing the same initial representation and interpreter, and
agreeing on every application query about $h_0$, that differ in $G_D$.
\end{propA1}

\begin{proof}
Set $U_0(r, d) = r$. Since the admissible history sets are finite and disjoint, a computable
decoder $\ell$ satisfies $\ell(d) = h$ on $D_h$. Set $U_1(r, d) = e(\ell(d))$ on admissible
histories and $U_1(r, d) = r$ otherwise. Both systems initially return
$v(r_0, x) = a_{h_0, x}$ for every $x$. For $d \in D_h$, $v\bigl(U_1(r_0, d), x\bigr) = a_{h,x}$,
so $G_D(S_{U_1})$ holds. Choose $h_1 \neq h_0$ and a query $x^*$ distinguishing them. On any
$d \in D_{h_1}$, $S_{U_0}$ returns $a_{h_0, x^*} \neq a_{h_1, x^*}$, so $G_D(S_{U_0})$ fails.
\end{proof}

\begin{corA1}
No evaluation restricted to application queries about $h_0$ distinguishes the two systems, even
if it exhausts $X$ or uses an arbitrary sequence of such queries.
\end{corA1}

\begin{proof}
Every answer is $v(r_0, x)$ in both systems, and those queries do not change the representation.
\end{proof}

\subsection*{A3\quad Specification does not entail discovery from observations}

Let $p_h$ be a mechanically decodable specification. Extend inputs with tags
$\mathrm{spec}(p_h)$ and $\mathrm{obs}(d)$ that identify input modes. The observation tag does not
disclose the correct target. Let $G_{\mathrm{spec}}$ require acquisition and correct application
for all specifications, and let $G_{\mathrm{obs}}$ be the corresponding condition for admissible
observations.

\begin{propA2}
Under the assumptions of Proposition~A1, $G_{\mathrm{spec}}$ does not entail
$G_{\mathrm{obs}}$.
\end{propA2}

\begin{proof}
Define $U_s$ to return $e(h)$ on $\mathrm{spec}(p_h)$ and retain the current representation on
observation inputs. The interpreter condition gives $G_{\mathrm{spec}}$. The witness
$(h_1, x^*)$ in Proposition~A1 gives failure of $G_{\mathrm{obs}}$ from $r_0$. A procedure
additionally using $\ell$ on observations succeeds in both modes.
\end{proof}

\subsection*{A4\quad The six-criterion realization}

Write the four features as $(t, g, s, c) \in \{0,1\}^4$. The six criteria are
\[
t \wedge c, \quad
(t \vee g) \wedge c, \quad
t \wedge \neg s, \quad
(t \vee g) \wedge \neg s, \quad
c \wedge \neg s, \quad
\mathbf{1}[\,t + g + c \geq 2\,].
\]
They define six distinct functions on the sixteen feature combinations. An interpreter evaluates
whichever criterion is represented. A learner can filter the finite candidate set against
observed labels.

An identifying observation set can be chosen by taking the feature combinations
$(0,0,0,1)$, $(0,1,0,0)$, $(0,1,1,1)$, $(1,0,0,0)$, $(1,0,1,1)$, and $(1,1,0,0)$ with the labels
provided by the target criterion. The resulting six-bit label signatures are different for every
criterion. Ten combinations remain unobserved. In particular, changing from the first to the
second criterion changes the classification of the unobserved combination $(0,1,0,1)$ as well as
the observed combination $(0,1,1,1)$. Thus acquisition has a consequence outside the identifying
sample, rather than merely replacing an observed answer.


\end{document}